\documentclass[journal]{IEEEtran}

\usepackage{graphicx} 
\usepackage{amsmath}
\usepackage{amssymb}
\usepackage{amsthm}

\usepackage{multirow}
\usepackage{booktabs}
\usepackage{comment}
\usepackage{makecell}

\usepackage{color}

\usepackage[normalem]{ulem}
\usepackage{comment}
\newtheorem{thm}{\bf Theorem}[section]
\newtheorem{lemma}{\bf Lemma}[section]
\usepackage{bbding}

\begin{document}
%
\title{Bilateral Asymmetry Guided Counterfactual Generating Network for Mammogram Classification}
%
%
%

\author{Chu-ran Wang*,
        Jing Li*,
        Fandong Zhang,
        Xinwei Sun\Envelope{},
        Hao Dong,
        Yizhou Yu,
        and Yizhou Wang\Envelope{} 
\thanks{* indicates equal contribution}
\thanks{\Envelope{} indicates corresponding author}
\thanks{Chu-ran Wang is with Center for Data Science, Peking University, Beijing, 100871, China, the work was done when she was an intern in Deepwise AI lab (e-mail: churanwang@pku.edu.cn).}
\thanks{Jing Li is with Dept. of Computer Science, Peking University, Beijing, 100871, China (e-mail: lijingg@pku.edu.cn).}
\thanks{Fandong Zhang is with Center for Data Science, Peking University, Beijing, 100871, China (e-mail: fd.zhang@pku.edu.cn).} 
\thanks{Xinwei Sun is with Microsoft Research Asia, Beijing, 100080, China (e-mail: xinsun@microsoft.com). }
\thanks{Hao Dong is with Center on Frontiers of Computing Studies, Dept. of Computer Science, Peking University, Beijing, 100871, China (e-mail: hao.dong@pku.edu.cn).} 
\thanks{Yizhou Yu is with Deepwise AI Lab, Beijing, 100080, China (e-mail: yizhouy@acm.org).}
\thanks{Yizhou Wang is with Dept. of Computer Science, Peking University, Beijing, 100871, China (e-mail: yizhou.wang@pku.edu.cn).}
}

%
%

\markboth{C. Wang \MakeLowercase{\textit{et al.}}: Bilateral Asymmetry Guided Counterfactual GAN for Mammogram Classification}%
{}
%



\maketitle

\begin{abstract}
Mammogram benign or malignant classification with only image-level labels is challenging due to the absence of lesion annotations. Motivated by the symmetric prior that the lesions on one side of breasts rarely appear in the corresponding areas on the other side, given a diseased image, we can explore a counterfactual problem that how would the features have behaved if there were no lesions in the image, so as to identify the lesion areas. We derive a new theoretical result for counterfactual generation based on the symmetric prior. By building a causal model that entails such a prior for bilateral images, we obtain two optimization goals for counterfactual generation, which can be accomplished via our newly proposed counterfactual generative network. Our proposed model is mainly composed of Generator Adversarial Network and a \emph{prediction feedback mechanism}, they are optimized jointly and prompt each other. Specifically, the former can further improve the classiﬁcation performance by generating counterfactual features to calculate lesion areas. On the other hand, the latter helps counterfactual generation by the supervision of classification loss. The utility of our method and the effectiveness of each module in our model can be verified by state-of-the-art performance on INBreast and an in-house dataset and ablation studies. 
\end{abstract}

\begin{IEEEkeywords}
Domain Knowledge, Bilateral Asymmetry, Counterfactual, Mammogram Classification
\end{IEEEkeywords}

%
\IEEEpeerreviewmaketitle

\section{Introduction}
Breast cancer is the leading cause of cancer death among women~\cite{siegel2019cancer}. The mammography-based \textbf{B}enign/\textbf{M}alignant \textbf{C}lassification (BMC) is considered to be an effective way for early breast cancer diagnosis.
Note that only the images with lesions need benign/malignant classiﬁcation. It is meaningless to tell the malignancy of healthy images since there are no lesions in them. Whether there are lesions in an image can be parsed from clinical reports. Since the existence of lesions is a necessary condition to be diagnosed as malignant, we are interested in benign/malignant classification for samples with lesions. The annotations of lesion areas require extra efforts such as bounding boxes of lesion areas \cite{dhungel2016automated,lotter2017multi,wu2018conditional,ribli2018detecting,tai2013automatic} and binary mask for segmentation \cite{chen2017dual}, which require expert domain knowledge and are costly and difficult to obtain. 
Therefore, addressing BMC with only image-level labels is valuable to clinical application. The key for BMC with the only image-level labels as supervision is to explore abnormal features for classification from a full mammogram image. This kind of abnormality can be expressed as masses, calcification clusters, structure distortions and their associated signs like skin retraction, skin thickening and so on. However, the high-intensity breast tissues in 2D image (as projection of the 3D organ) may partially obscure the lesions, making the problem more challenging.

To solve this problem, existing works mainly utilize specific rules or attention modules for feature selection, such as the selected local features with the maximum response or largest prediction score~\cite{zhu2017deep}, and select the most discriminative region via the proposed attention branch supervised by a classification signal~\cite{fukui2019attention, zhou2016learning}. The common problem for these methods lie in failing to take advantage of mammogram domain knowledge, which can be very valuable for lesion localization.

\begin{figure*}
\begin{center}
\includegraphics[height=0.53\linewidth]{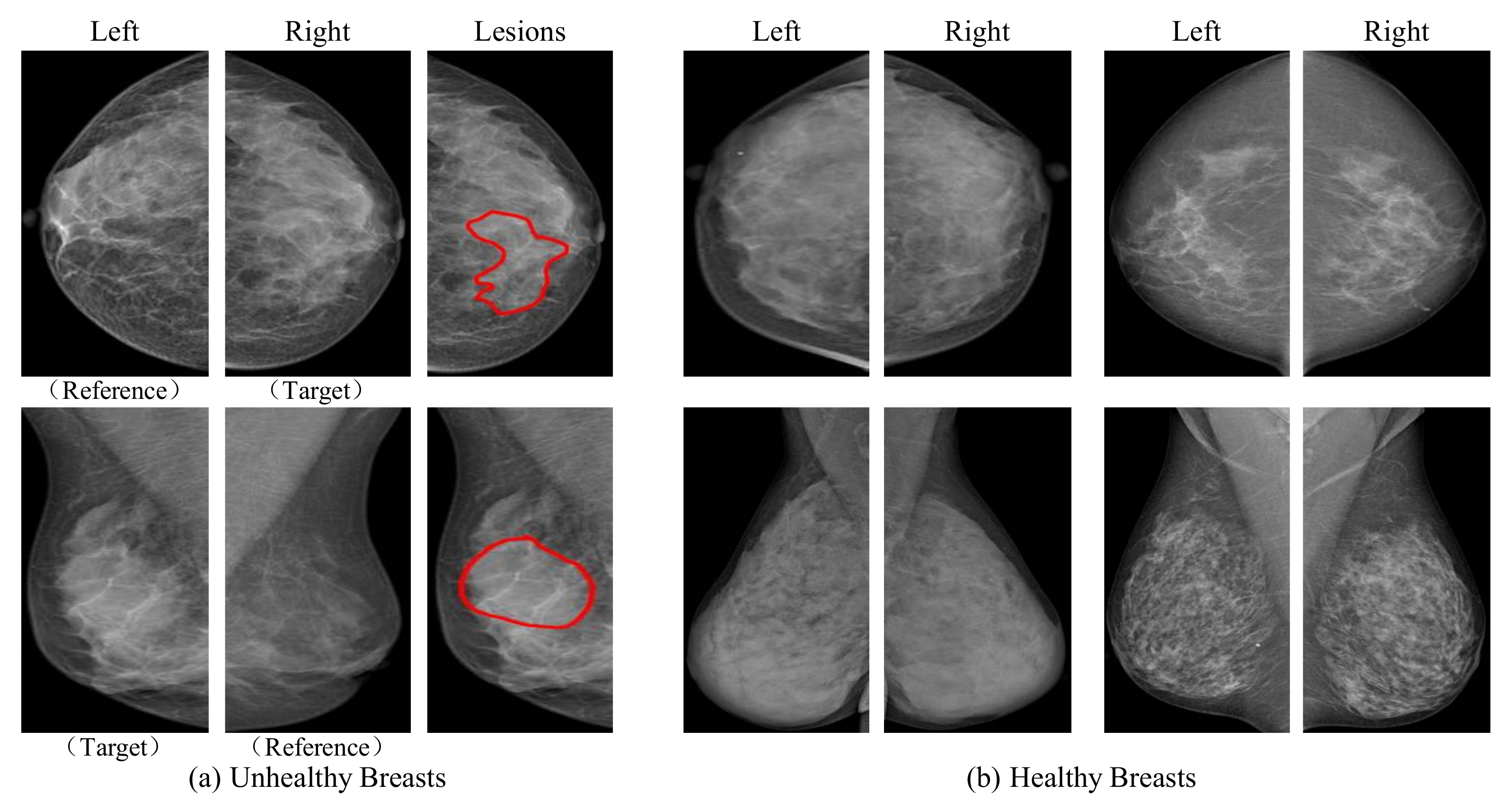}
\end{center}
\centering
\caption{(a) Two cases to show how the unhealthy breasts look asymmetrical.  (b) Illustrations of that healthy breasts are roughly bilaterally symmetrical, with patterns and appearance (e.g., structure, distribution, density, and morphology) of breast tissues can be very diverse among them.}
\label{Asymmetryshow}
\vspace{-0.1cm}
\end{figure*}

One important mammogram domain knowledge is ``Anotomical Symmetry", which has been authenticated by BI-RADS standard of American College of Radiology \cite{sicklesacr}. It refers to that the lesion area in the target image (denoting the image from target side to be classified) of breast rarely appears in the corresponding area in the reference image (denoting the image of the opposite side). There is no lesion in the corresponding area on the other side, as shown in Fig.~\ref{Asymmetryshow}. Due to such a prior, the radiologists commonly compare bilateral breasts to find the asymmetric regions for further diagnosis.

Such a prior naturally motivates the counterfactual generation question: \emph{what would the features of the target image have been looked like had lesions removed, given observed target image with lesions and the reference image that is lesion-free in the corresponding area?} 
After such counterfactual features being generated, the residue between the original target features and the counterfactual one incorporates the information of lesion hence can provide an informative and interpretable guidance for BMC. The answer to the above question is via constructing a structural causal model \cite{pearl2009causality} in which the counterfactual learning is well defined.
Specifically, a structural causal model (SCM) is proposed that introduces latent bilateral variables for generating bilateral images. To depict the bilateral symmetry, we further introduce a hidden confounder (including DNA, environment, etc.) that generates such bilateral features via the same causal mechanism, naturally leading to an inspiring conclusion: the target features of counterfactual generation share the same distribution (i) with the reference features in lesion areas and (ii) with the target features in lesion-free areas, namely \emph{counterfactual constraints}. Based on such a theoretical finding, we propose a novel \textbf{C}ounterfactual \textbf{G}eneration \textbf{N}etwork (CGN). Note that pixel-to-pixel registration between bilateral images is challenging due to unpleasant spatial distortion during image capturing and imperfect anatomical symmetry, we apply counterfactual generation in feature level motivated by~\cite{liu2019unilateral}. Moreover, it achieves faster training speed without losing prediction power. This is also the reason why many domain adaptation methods work on feature space.
Our CGN iteratively optimizes counterfactual generation under counterfactual constraints and lesion-area estimation via an attention-based prediction feedback mechanism. Both the lesion-area estimation and counterfactual generation are optimized jointly and prompt each other, supervised by classification loss. Finally, the residual features that incorporate the accurate lesion information, and the original target features which encodes the contextual information, are concatenated for the final classification.

In contrast to existing GAN-based works \cite{zhu2017unpaired,siddiquee2019learning,schlegl2017unsupervised} for counterfactual generation, our method is endowed with a theoretical guarantee regarding the counterfactual distribution \cite{chernozhukov2013inference} by exploiting the symmetric prior. Specifically, AnoGAN \cite{schlegl2017unsupervised} learns the latent space of healthy data and assumes that the lesions can not be reconstructed within such latent space. Therefore the areas with large reconstruction errors are more likely to be lesions. Its performance highly relies on how well the healthy data modeled. However, in our mammogram application, the glandular structure and characterization of healthy images can be very diverse. Sometimes the healthy pattern can even be similar to lesions, as shown in Fig~\ref{Asymmetryshow}. Thus it is challenging to model healthy patterns well and distinguish the lesions at the same time using only healthy data.
While another cycle consistency loss based method targets on lesion removal \cite{zhu2017unpaired, siddiquee2019learning}. Although these methods can utilize the lesion information by learning a back translation (i.e., from the counterfactual to the original), they also suffer from the healthy modeling problem in the forward translation (i.e., from the original to the counterfactual). What is more, these methods all assume that the translated data can be translated back to the original data~\cite{hu2019mask,nizan2019breaking}. In our application, it means the back translation network should be able to model the location and appearance of the removed lesion. However, mammogram lesions can appear anywhere, i.e., the location of the lesions is unpredictable. Therefore, it is an ill-posed problem to translate the counterfactual data back to the corresponding original data perfectly.

In this paper, we introduce symmetry prior to counterfactual learning to propose a bilateral asymmetry guided counterfactual generating network (CGN), improving the performance of mammogram classification. Instead of learning from healthy images, our CGN applies counterfactual generation conditioning on the bilateral information. Based on the symmetry prior, we formulate the generated counterfactual features and estimated lesion areas together by counterfactual constraints: being similar distribution with the reference features in lesion areas and maintaining most of the information of target features in lesion-free areas. Therefore, we first apply a deep generator with AdaIN~\cite{huang2017arbitrary} mechanism to provide the feature generation ability. Then we design a prediction feedback mechanism to help estimate the lesion areas. Meanwhile, an adversarial reference loss, a feedback triplet loss, and an auxiliary negative embedding loss are proposed to encourage the generated features to satisfy the above counterfactual constraints. Both the lesion-area estimation and counterfactual generation are optimized jointly and prompt each other. Further, we get the residual features by computing the difference between the generated counterfactual features and target features. Finally, we aggregate the residual features together with the target features for the final classification.

We evaluate the proposed method on a public dataset INBreast~\cite{Moreira2012INbreast} and an in-house dataset. Our CGN achieves an area under the curve (AUC) of 91.1\% on INBreast and 78.1\% on the in-house dataset, which largely outperforms the representative methods. To summarize, our contributions are mainly three-fold:
\begin{enumerate}
    \item First, for benign or malignant classification with only image-level labels, we propose a novel counterfactual-based method to learn the healthy features of the target image, which can help localize the lesions to prompt further classification;
    \item Second, we draw the bilateral symmetry prior to the molybdenum target images into the counterfactual generation for learning counterfactual features reasonably and effectively;
    \item Third, we achieve state-of-the-art performance for mammogram classification on both the public and in-house datasets.
\end{enumerate}


\section{Related Work}

\subsection{BMC with only image-level labels} Previous approaches that can be used to address BMC with only image-level labels without any extra annotations are roughly categorized into two classes: (i) the attention-based methods, \emph{e.g.,} Zhu \textit{et al.}~\cite{zhu2017deep}, Zhou \textit{et al.}~\cite{zhou2016learning} and Fukui \textit{et al.}~\cite{fukui2019attention}; (ii) the simple multi-view fusion methods, \emph{e.g.,} Wu \textit{et al.}~\cite{wu2019deep}. For the class (i), they extend a response-based visual explanation model with an attention module or specific rules. However, they all ignore medical domain knowledge which is valuable for BMC and are fragile when facing dense breasts without learning from bilateral information. For the class (ii), since the bilateral breasts are not pixel-to-pixel symmetry, simple multi-view fusions can be very sensitive to bilateral misalignment. Motivated by above, we take advantage of domain knowledge and design CGN to improve BMC.

\subsection{Counterfactual Generation} Existing GAN-based models for counterfactual generation can be roughly categorized into two classes: (i) healthy modeling methods, \emph{e.g.,}AnoGAN~\cite{schlegl2017unsupervised} and (ii) cycle consistency based methods, \emph{e.g.,}~CycleGAN~\cite{zhu2017unpaired}, Fixed-point GAN~\cite{siddiquee2019learning}. For class (i) that learns to model the pattern of healthy data, 
they suffer from unstable result due to large diversity of glandular structure and characterization of healthy images which are hence difficult to model.
Another line of work, i.e., class (ii), uses cycle consistency loss to incorporate bi-directed translation: forward translation (from the original to the counterfactual) and back translation (from the counterfactual to the original). These methods suffer from two problems: a) the healthy modeling problem for forward translation, similar to class (i); b) the ill-posed problem for back translation since the location and appearance of the removed lesion is diverse and unpredictable. In contrast to existing works, our method learns healthy pattern by exploiting symmetric prior, so as to avoid the problems mentioned above and hence be able to achieve more robust counterfactual generation result.

\section{Methodology}
\label{method}

\begin{figure*}[t]

\begin{center}
\includegraphics[height=0.45\linewidth]{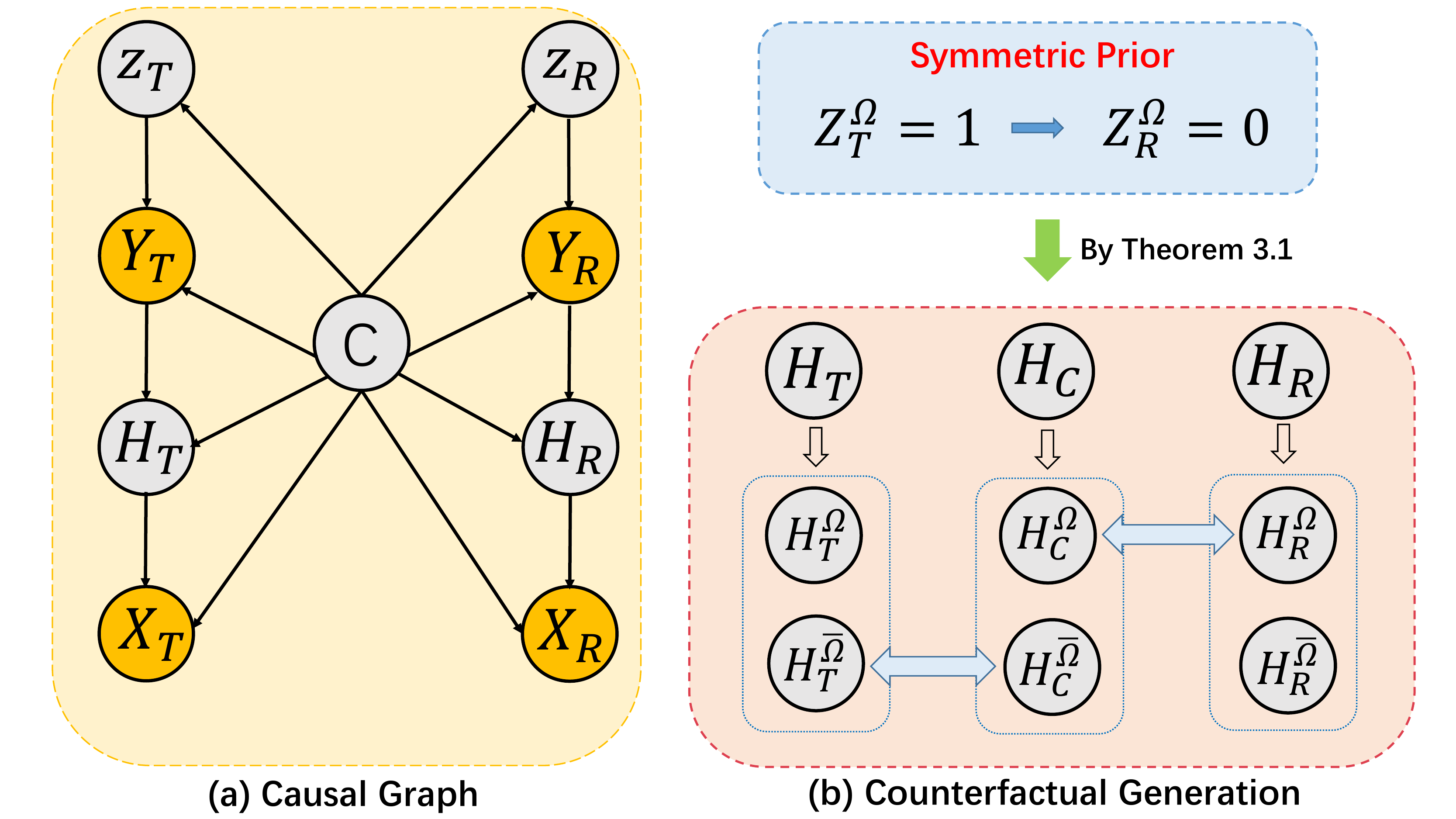}
\end{center}
\caption{(a): Our causal graph with observed variables marked by yellow and unobserved variables marked by gray. For notations, $C$ denotes the DNA, growth environment that can explain the common properties shared between $X_T$ and $X_R$; $Z_R,Z_T$ denote lesion states ( $Z_{u=T,R}^{\Omega} = 1$ if there are lesions in $\Omega$; and $=0$ if not ); $H_R, H_T$ respectively denote the hidden features of the image. (a) is mathematically expressed in our Eq~\eqref{eq:scm}. (b): Our counterfactual learning framework, motivated by symmetric prior (as shown in the top blue box). Our theoretical result (theorem~\ref{thm1}) is illustrated in the bottom orange box, in which the $H_C$ denotes the counterfactual result of the target side with the removal of lesion areas, i.e.,  the counterfactual result of $H_T$ under counterfactual event $Z_T = 0$. The blue arrows denote "distributionally equivalence". As shown, the distribution of $H_C^{\Omega}$ is the same with $H_R^{\Omega}$, described by Eq.~\eqref{eq1}; the distribution of $H_{C}^{\overline{\Omega}}$ is the same with $H_{T}^{\overline{\Omega}}$, described by Eq.~\eqref{eq2}.}
\label{fig:causal_graph}
\end{figure*}

\textbf{Problem Setup and Notations} The goal of mammogram benign or malignant classification is to learn classifier $f: \mathcal{X} \to \mathcal{Y}_T$ that predicts the disease label of target side $X_T$, where $\mathcal{X} := (\mathcal{X}_T, \mathcal{X}_R)$ ($\mathcal{X}_T, \mathcal{X}_R \subset \mathbb{R}^d$) denotes the input space of bilateral breast images with $T$ denoting the target side of bilateral breast image and $R$ correspondingly denoting the other side, a.k.a, reference side, and $\mathcal{Y}_T:=\{0,1\}$ denotes the disease label of the target side (1 denotes malignant and 0 denotes benign). To achieve this goal, we are given training data $\{(x^i_T,x^i_R,y^i_T)\}_{i \in [N]}$ ($[N] := \{1,...,N\}$ for any integer $N > 0$).
During test stage, our goal is to predict $y_t$ for a new instance $x = (x_T,x_R) \in \mathcal{X}$. 

\subsection{Counterfactual Learning}
\textbf{Symmetric Prior~\cite{sicklesacr}} \emph{For a paired image data, if the target image contains lesions, the corresponding symmetrical area in the reference image has almost certainly no lesions.}

This symmetric prior provides a guidance for localizing lesion areas, as a residue of the feature of target image subtracting the one with the removal of corresponding lesions. The generation of the latter image, which can leverage the information of the reference features due to symmetric prior, is a counterfactual problem, i.e., \emph{what would the features of target image have been looked like had lesions removed, given observed target image with lesions and the reference image that is lesion-free in the corresponding area?} Such a counterfactual problem has been well-defined and explored in the framework of (Structural) Causal Model (SCM) \cite{pearl2009causality} that describes the generating process of observational variables, with assumptions entailed in the corresponding causal graph.

To describe bilateral images, we propose a SCM that introduces a hidden common factor (denoted as $C$ which can refer to DNA, growth environment, etc.) that generates bilateral variables, which depicts our symmetric prior, as shown in Fig.~\ref{fig:causal_graph} (a). Besides, our SCM incorporates bilateral latent features, denoted as $H_{U=T,R}$ ($T$ denotes target side and $R$ denotes reference side), as abstraction/concepts of bilateral images. Such bilateral features, which are affected by $C$ and disease status ($Y_{U=T,R}$) that is determined by lesion status $Z_{U=T,R}$. The distribution of these variables are assigned by the following structural equations: $ C = f_C(\epsilon) \to $
\begin{small}
\begin{align}
    \begin{cases} Z_T=f_{Z_T}(C) \\
    Z_R=f_{Z_R}(C) \end{cases} \to \begin{cases} Y_T=f_{Y_T}(C, Z_T) \\ Y_R=f_{Y_R}(C, Z_R) \end{cases} \to
    \notag
\end{align}
\end{small}
\begin{small}
\begin{align}
\label{eq:scm}
    \begin{cases} H_T=f_{H_T}(C, Y_T) \\  H_R=f_{H_R}(C, Y_R)\end{cases} \to \begin{cases} X_T=f_{X_T}(C, H_T) \\ X_R=f_{X_R}(C, H_R). \end{cases}
\end{align}
\end{small}
Equipped with such a SCM, we can mathematically formulate the symmetric prior as $Z_T^{\Omega} = 1 \to Z_R^{\Omega} = 0$, with $\Omega$ denoting the lesion areas of the target image $X_T$; and counterfactual generation problem as 
$H_{{T}_{(Z_{T}=0)}}^{\Omega}(c)$ that can be read as the value of $H_{T}$ on $\Omega$ in situation $C=c$ had $Z_{T}^{\Omega} =0$ \cite{pearl2009causality}. Since the situation $C = c$ is induced by the factual event $\{H_{T}^{\Omega} = h_t, Z_{T}^{\Omega} = 1\}$, our counterfactual distribution 
can be denoted as $P(H_{{T}_{(Z_{T}^{\Omega}=0)}}^{\Omega}=h|H_{T}^{\Omega}=h_t, Z_{T}^{\Omega} = 1)$. Under our SCM and the symmetric prior, we have following results for counterfactual generation: 



\begin{thm}\label{thm1}
Under the symmetric prior, the structural equation model defined in Eq.~\eqref{eq:scm} for Fig.~\ref{fig:causal_graph} (a) has the following results for counterfactual distribution of target features: 
\begin{align}
     P(H_{{T}_{(Z_T^{\Omega}=0)}}^{\Omega}=h|H_T^{\Omega}=h_t, Z_T^{\Omega}=1)\notag\\
    = P(H_R^{\Omega}=h_r|H_T^{\Omega}=h_t, Z_T^{\Omega}=1) \label{eq1} \\
     P(H_{{T}_{(Z_T^{\overline{\Omega}}=0)}}^{\overline{\Omega}}=h|H_T^{\overline{\Omega}}=h_t, Z_T^{\overline{\Omega}}=0) \notag\\
      = P(H_T^{\overline{\Omega}}=h_t|H_T^{\overline{\Omega}}=h_t, Z_T^{\overline{\Omega}}=0) \label{eq2},
\end{align}

\end{thm}

The proof of Theorem~\ref{thm1} is shown in our appendix. This theorem implies that the generated counterfactual features should be equal (i) to reference features in lesion areas, (ii) to target features in lesion-free areas, which leads the following two goals for the counterfactual generation: 
\begin{equation}
  \underset{\theta}{min}\quad D(P_{\theta}(H_{T_{(Z_T^{\Omega}=0)}}^{\Omega}), P_{\theta}(H_R^{\Omega}))  
  \label{eq:constraints1}
\end{equation}
\begin{equation}
    \underset{\theta}{min}\quad D(P_{\theta}(H_{T_{(Z_T^{\overline{\Omega}}=0)}}^{\overline{\Omega}}), P_{\theta}(H_T^{\overline{\Omega}}))
    \label{eq:constraints2}
\end{equation}

where $D$ denotes generalized distance measure, \emph{e.g.,} KL divergence. 
With such counterfactual learning, it is expected that the lesion areas, as the subtraction of counterfactual generation of $H_T$ (with lesions removed) from original $H_T$, can be detected precisely and hence can lead to accurate classification performance. To achieve the above two goals, we propose a counterfactual generating network (CGN), which cooperatively localizes the lesion areas and achieve counterfactual generation simultaneously. We explain the CGN in details in the subsequent section.

\begin{figure*}[t]
\begin{center}
    \includegraphics[height=0.5\linewidth]{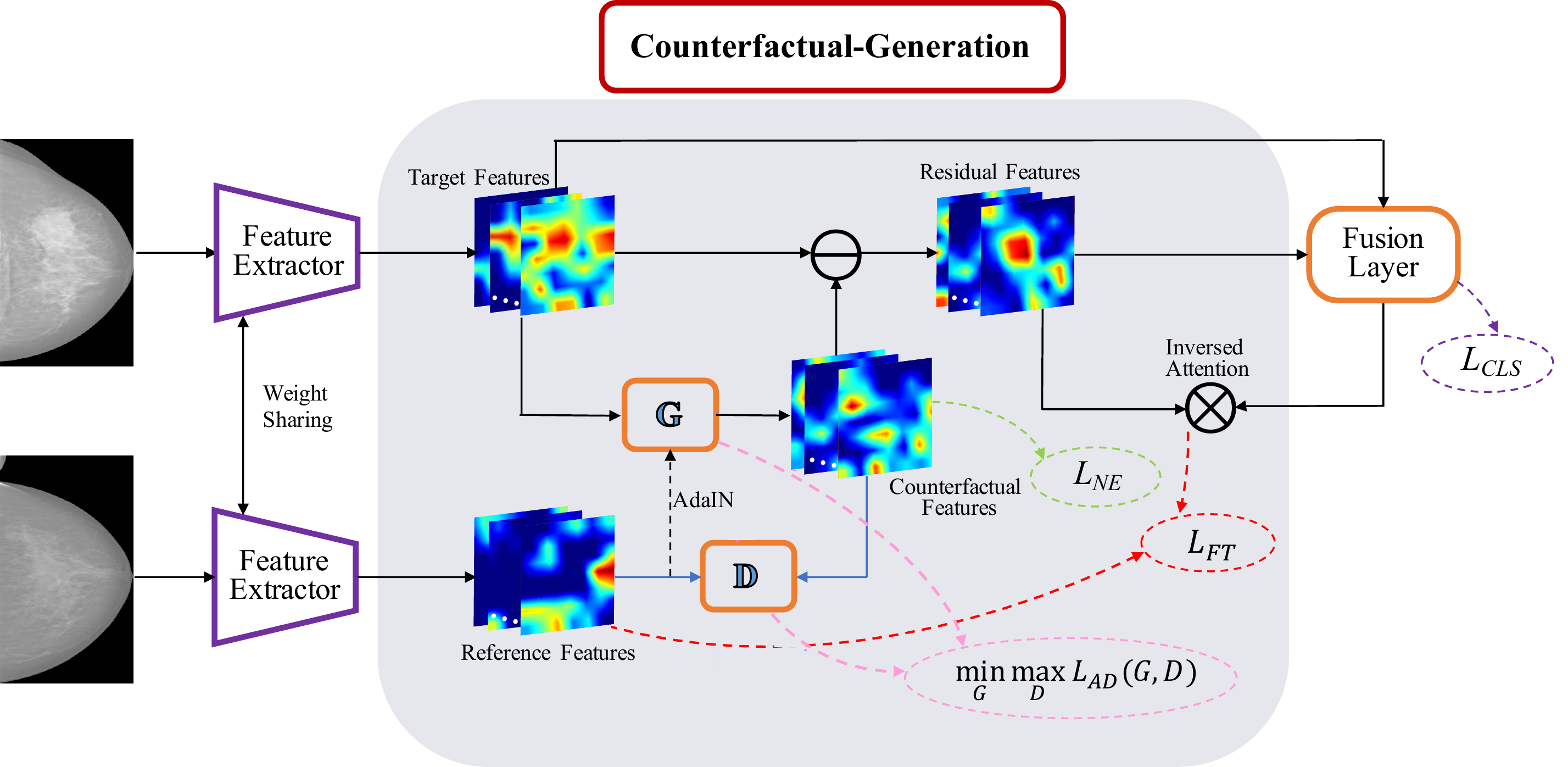}
\end{center}
    \centering
    \centering\caption{The schematic overview of CGN. First, two feature extractors with weight sharing extract the features for input paired target and reference images, respectively. Then the bilateral features are processed by AdaIN mechanism and fed into the generator $G$ to generate the counterfactual features. The counterfactual features are constrained by adversarial learning with a feedback triplet loss $\mathcal{L}_{FT}$, and a negative embedding loss $\mathcal{L}_{NE}$. Then, the residual features are obtained by computing the difference between the target features and counterfactual features. Finally, the residual features are fed into a \emph{Fusion} network with target features and outputs prediction of benign/malignant.
}
\label{fig:overview}
\end{figure*}

\subsection{Counterfactual Generating Network (CGN)}
\label{sec:algorithm}
As illustrated in Fig.~\ref{fig:overview}, our counterfactual generation network for mammogram classification contains the following steps: (i) generation of target and reference features $H_T$ and $H_R$ from images $X_T$ and $X_R$, via a feature extractor chosen from backbone network, \emph{e.g.} AlexNet~\cite{krizhevsky2012imagenet}, ResNet~\cite{hermans2017defense}, (ii) a \textbf{\emph{counterfactual generation module}} is designed to generate counterfactual features $H_C$ from both $H_T$ and $H_R$, (iii) \textbf{\emph{a classification module}} is designed to predict malignant/benign, with aggregated $H_C$ and $H_T$ as input. To accurately identify $\Omega$ for generating $H_C$ in step (ii), a \textbf{\emph{prediction feedback mechanism}} and a set of \textbf{\emph{counterfactual constrains}} motivated by Eq.~\eqref{eq:constraints1} and ~\eqref{eq:constraints2} are designed. In what follows, we will explain the above mechanisms in more details.

\textbf{Counterfactual Generation Module} The Adaptive Instance Normalization (AdaIN)~\cite{huang2017arbitrary}, which has been proved to be effective for style transfer tasks, is adopted as the generator $G$ (as shown in Fig.~\ref{fig:overview}) for counterfactual generation, with $H_T$ as content and $H_R$ as style in our case: 
\begin{equation}
    AdaIN(H_T, H_R)=\sigma(H_R)\left(\frac{H_T-\mu(H_T)}{\sigma(H_T)} \right)+\mu(H_R)
\end{equation}
with $\mu(\cdot)$ and  $\sigma(\cdot)$ denoting the mean and standard variance function. As suggested by~\cite{huang2017arbitrary}, an interpolated $H_T$ and AdaIN are fed into a generator network containing nine residual blocks to generate counterfactual features $H_C$: 
\begin{equation}
    H_C = G((1-\alpha) * H_T + \alpha * AdaIN(H_T, H_R))
\end{equation}
where $\alpha$ is a hyper-parameter of the interpolation weight.

\textbf{Classification Module} The residual features (entailing lesion information) obtained by $H_T-H_C$ and $H_T$ (with additional contextual information which is showed useful for the medical image inference~\cite{Amores2005RetrievalOI} besides lesion-related information we obtained) are fed into a classifier in a concatenated way. This classifier, which implements a convolutional block as FusionLayer to obtain the fused features, is trained via commonly used cross-entropy loss:
\begin{align}
    \label{cls_loss}
    \mathcal{L}_{CLS}(G) = - (Y_T^{gt}*\log Y_T(H_T-G(H_R,H_T),H_T) \notag\\+ (1-Y_T^{gt})*\log (1- Y_T(H_T-G(H_R,H_T),H_T))),
\end{align}

where $Y_T(H_T-G(H_R,H_T),H_T)$ (with $H_C=G(H_R,H_T)$) is the classification probability. 

\textbf{Prediction Feedback Mechanism}
This mechanism is to estimate the lesion areas $\Omega$ for better counterfactual generation. Specifically, we use the attention map, in which the locations with higher value implies higher lesion probabilities, as final estimation of $\Omega$. Such an attention map is calculated by normalization/softmax following the class activation map (CAM)~\cite{zhou2016learning}, i.e., $\Omega_s=softmax(CAM)$. $\Omega_s$ is the corresponding prediction probabilities of being lesions at each position.

\begin{table*}
\centering
\scriptsize
\begin{tabular}{c|cccc}
\hline
Method & AUC (a) & AUC (b) & AUC (c) & AUC (d)  \\
\hline
Pretrained CNN~\cite{dhungel2016automated} & 0.690 & $-$ & $-$ & $-$ \\
Pretrained CNN+Random Forest~\cite{dhungel2016automated} & 0.760 & $-$ & $-$ & $-$ \\
Vanilla AlexNet, Zhu \textit{et al.} \cite{zhu2017deep} & 0.790 & $-$ & $-$ & $-$ \\
Zhu \textit{et al.} \cite{zhu2017deep} & 0.890 & $-$ & $-$ & $-$   \\ \cline{1-5}
Vanilla* & 0.820 & 0.827 & 0.780 & 0.697 \\
AnoGAN \cite{schlegl2017unsupervised}* & 0.803 & 0.796 & 0.774 & 0.720 \\
Fixed-Point GAN \cite{siddiquee2019learning}* & 0.835 & 0.837 & 0.805 & 0.734 \\
CycleGAN \cite{zhu2017unpaired}* & 0.852 & 0.838 & 0.808 & 0.741\\
Wu \textit{et al.} \cite{wu2019deep} & 0.863  & 0.860 & 0.810 & 0.723 \\
Zhu \textit{et al.} \cite{zhu2017deep}* & 0.860 & 0.862 & 0.830 & 0.720 \\
Vanilla*+GAP \cite{zhou2016learning}* & 0.857 & 0.827 & 0.780 & 0.718 \\
Vanilla*+ABN \cite{fukui2019attention}* & 0.858  & 0.846 & 0.814 & 0.723 \\\cline{1-5}
Proposed Method & \textbf{0.910}  &  \textbf{0.911} & \textbf{0.885} & \textbf{0.781}\\
\hline
\end{tabular}
\caption{AUC evaluation of comparative experiments on
(a) INBreast + Alexnet (mass); (b) INBreast + Resnet50 (mass); (c) INBreast + Resnet50 (mixed lesions); (d) In-house + Alexnet (mixed lesions); Note that the '*' means our re-implementation. The '-' means there are no official report results.
}
\label{tab:comparison_all}
\end{table*}

\textbf{Counterfactual Constraints} Since the direct optimization of Eq.~\eqref{eq:constraints1} and~\eqref{eq:constraints2} can be intractable/unstable for general distance measure $D$ such as KL-divergence, we adopt the adversarial learning strategy~\cite{goodfellow2014generative}. For optimization of Eq.~\eqref{eq:constraints1}, GAN generates similar features from the whole reference image and can constrain our desired features be the same as the references in lesion areas.
Specifically, a Discriminator $D$ (learns to classify $H_C$ and $H_R$) and a Generator $G$ (fools the discriminator) are designed and trained in a competing way:
\begin{align}
\label{eq:lossgd}
& \min_G \max_D \mathcal{L}_{AD}(G,D) := \notag\\
& \log\left ( D\left ( H_R \right ) \right ) + \log\left (1- D\left ( G(H_T,H_R) \right ) \right ).
\end{align}
However, the generated features through GAN loss are undesired features in lesion-free areas. For optimization of Eq.~\eqref{eq:constraints2}, we use a prediction feedback mechanism to localize lesion areas. One intuitive way to use feedback mechanism is constraining generated features be the same as the target features in lesion-free areas directly or only constrains the generated features be the same as the reference features in lesion areas in discriminator. However, motivated by~\cite{schroff2015facenet} triplet loss can be better than such designs. They will suffer from slow convergence and falling into local minimum easily and we analysis and evaluate such variant methods in Sec~\ref{sec:ablation study}. Thus, we propose a feedback triplet loss to minimize the distance between the target features $H_T^{\overline{\Omega}}$ and counterfactual features $H_C^{\overline{\Omega}}$ in lesion-free areas, which is measured by target-counterfactual distance $d_{tc}$ by weighted mean square error:
\begin{equation}
    d_{tc}=\frac{\sum_{i}^{h}\sum_{j}^{w}(1-{\Omega}^{ij})\left \| H^{ij}_T-H^{ij}_C\right \|_{2}^{2}}{h\times{w}-1}
    \label{loss:dtc}
\end{equation}, where $h$ and $w$ denote the height and width of CAM respectively.
Motivated by minimization of distance between $H_C$ and $H_R$ enforced by Eq.~\eqref{eq:lossgd}, we choose a $d_{rc}$ between $H_R$ and $H_T$ as an adaptive reference to minimize $d_{tc}$. The $d_{rc}$ is measured by chamfer distance~\cite{achlioptas2017learning} to endure the misalignment, and is defined by
\begin{equation}
    d_{rc}=\frac{\sum_{i}^{h}\sum_{j}^{w}(\underset{u,v}{min}\left \| H^{ij}_{R}-H^{uv}_{C}\right \|_{2}^{2} + \underset{u,v}{min}\left \| {H^{ij}_{C}}-H^{uv}_{R}\right \|_{2}^{2})}{2\times{h}\times{w}}
\end{equation}
Therefore, the feedback triplet loss is defined as:
\begin{equation}\label{eq:ftl}
\mathcal{L}_{FT}(G) = \max\left \{ 0, d_{tc} + \beta - d_{rc} \right \}
\end{equation}
The triplet loss makes $H_C$ be closer to $H_T$ than $H_R$ in terms of the lesion-free areas. Further the GAN loss makes the distance between $H_C$ and $H_R$ be close in the lesion areas. Based on the cooperation of GAN loss and the triplet loss, the generated $H_C$ satisfies Eq.~\eqref{eq:constraints1} and~\eqref{eq:constraints2}.
Besides, ${L}_{FT}(G)$ as a margin term can avoid learning identity mapping from $H_T$ to $H_C$ during minimizing ${L}_{FT}(G)$. Catering misalignment is not needed for $d_{tc}$ since $H_C$ is for the “target" and hence perfectly aligned with $H_T$ in pixel-wise. 

Besides, since the lesion regions of $H_T$ have been removed in $H_C$, the $H_C$ must also be non-malignant. Such a knowledge can be reflected via auxiliary negative embedding loss as a constraint:
\begin{equation}\label{eq:lne}
\mathcal{L}_{NE}(G) = -\log(1-p_m(H_C)),
\end{equation}
where $p_m(H_C)$ denotes the malignant probability of $H_C$.

\textbf{Joint Optimization} The final loss is combination of the losses defined in Eq.~\eqref{cls_loss},~\eqref{eq:lossgd},~\eqref{eq:ftl} and~\eqref{eq:lne}: 
\begin{align}\label{eq:loss}
\min_G \max_D \mathcal{L}(G,D) :=\sum _k \big \{\mathcal{L}_{AD}^{k}(G,D) + \mathcal{L}_{NE}^{k}(G) \notag\\+ \mathcal{L}_{FT}^{k}(G) + \mathcal{L}_{CLS}^{k}(G) \big \}.
\end{align}
, where $k$ denotes sample index, that is, we calculate corresponding losses for each sample and derive the final joint loss.
By optimizing the loss $\mathcal{L}(G,D)$, these modules can be optimized cooperatively and compatibly: the counterfactual generation helps discover the lesions for classification; on the other hand, the classiﬁcation module helps counterfactual generation in a supervised way. The effect of these modules can be validated by our ablation study, which are explained detailedly in the next section.

\section{Experiments}

\subsection{Implementation Details}

Mammogram images are commonly stored using a 14-bit DICOM format. A simple linear mapping is used to convert them into 8-bit gray images. Then, the Otsus method~\cite{otsu1979a} is used for breast region segmentation and background removal. The segmented images are resized into $224\times224$ and fed to networks. We implement all models with PyTorch. The models are initialized by ImageNet pre-trained weights for a fair comparison with the representative method~\cite{zhu2017deep}. For training, we use Adam optimization with a learning rate of $5e-5$ and train for 50 epochs. For all experiments, we select the best model on the validation set for testing. Both target and reference features are extracted from the last convolution layer.

\begin{table}
\centering
\begin{tabular}{c|c|c}
\hline
\textbf{Methodology} & \textbf{Top-1 error}(b) & \textbf{Top-1 error}(d)\\
\hline
ResNet50\cite{he2016deep} & 0.635 & 0.727 \\
AnoGAN \cite{schlegl2017unsupervised}* & 0.684 & 0.789 \\
Fixed-Point GAN \cite{siddiquee2019learning}* & 0.646 & 0.737 \\
CycleGAN \cite{zhu2017unpaired}* & 0.632 & 0.667 \\
Wu \textit{et al.}~\cite{wu2019deep}* & 0.627 & 0.650 \\
ABN~\cite{fukui2019attention} & 0.632 & 0.722 \\
Zhu \textit{et al.}~\cite{zhu2017deep}* & 0.627 & 0.625\\
\hline
Proposed Method & \textbf{0.421} & \textbf{0.455}\\
\hline
\end{tabular}
\caption{Top-1 localization error on (b) INBreast dataset for mass classification with Resnet50; (d) INBreast dataset for mixed-lesion classification with Resnet50. }
\label{tab:localization}
\end{table}

\subsection{Datasets}

We evaluate our method on the public INBreast dataset~\cite{Moreira2012INbreast} due to its high quality compared to other public datasets \cite{zhu2017deep} and an in-house dataset. 
The INBreast dataset contains 115 cases and 410 mammograms. INBreast provides each image a BI-RADS result as image-wise ground truth and we use the same process as Zhu \textit{et al.} \cite{zhu2017deep}. (malignant if BI-RADS $>$ 3; benign otherwise). Our experimental setting in INBreast is all the same as Zhu \textit{et al.} \cite{zhu2017deep} who uses 100 mammogram images with masses and reports image-wise malignant classiﬁcation performance. We discard 9 of them for lack of contralateral images in the same task. The remaining 91 images all have opposite sides, i.e. 91 pairs for mass malignancy classification. We consider two settings: the mass-lesion image classification and mixed-lesion classification in which the lesion can be masses, calcification clusters and distortions. First, we follow~\cite{zhu2017deep} and select only the images containing masses for mass malignancy classification. In particular, we discard 9 images for the absence of the reference image. Second, to be generalized, we also evaluate mixed-lesion malignancy classification including masses, calcification clusters, or distortions. We use five-fold cross-validation for evaluation and area under the curve (AUC) for measurement.

The in-house dataset contains 2500 images, where 1303 images contain image-level malignant annotations. The dataset contains 589 only masses, 120 only suspicious calcifications, 34 only architectural distortions, 197 only asymmetries and 363 multiple lesions from 642 patients. All these 1303 images have opposite sides, i.e. 1303 pairs (Note that the target image A with a malignancy annotation is paired with B,  counting as one pair. Meanwhile, if B also has a malignancy annotation, conversely B can be the target and A can be the reference, counting as another one pair). We randomly divide the dataset into training, validation and testing sets by the proportion of $8:1:1$ in patient-wise.

\subsection{Experiment settings} 

To fairly compare our method with others in a more general way, we implement AlexNet as backbone on both INBreast (for mass malignancy classification) and in-house dataset (for mixed-lesion malignancy classification). And we implement Resnet50 as backbone on INBreast (for both mass malignancy classification and mixed-lesion malignancy classification). 

\subsection{Bilateral Distribution Verification} 

In this section, we verify the correctness of our symmetric prior assumption which is motivation of our proposed framework. Specifically, we choose 1,000 unhealthy couples of the bilateral images, each of which contains at least one lesion from the in-house dataset. Then for comparison, we choose another 1,000 healthy couples. We do not use the public INBreast dataset since there are few healthy couples in it. 
To measure the image distribution distances, we use Fréchet Inception Distance (FID) \cite{heusel2017gans}, which has been used to evaluate medical images \cite{haarburger2019multiparametric,malkiel2019conditional}.
After calculating FID value of healthy set $D^H$ and the unhealthy set $D^U$, we conduct \emph{Hypothesis Testing} with the null hypothesis $H_0$ and althernative hypothesis $H_1$ defined as: 
\begin{equation}\label{eq:ttest}
H_0:\mu(D^H)>=\mu(D^U) \  \ H_1:\mu(D^H)<\mu(D^U).
\end{equation}
We obtain a p-value of $0.014<0.05$, which provides an evidence for us to reject $H_0$, i.e., the bilateral distribution distance of unhealthy cases is larger than healthy cases significantly. This result can be regarded as a manifestation of our symmetric prior assumption.

\subsection{Experimental Analysis}
\textbf{Compared Baselines for Malignancy Classification.} We conduct our experiments on both Mass malignancy classification(the 2nd and the 3rd columns of Table~\ref{tab:comparison_all}) and Mixed-lesion Malignancy classification(the last two columns of Table~\ref{tab:comparison_all}). 

The first four lines in Table~\ref{tab:comparison_all} summarize the official results of the representative methods.
To be fair, we compare the results with the backbone of AlexNet~\cite{krizhevsky2012imagenet} and ResNet50~\cite{hermans2017defense} separately. Due to the slightly difference in the number of images used by reference absence, for a fair comparison, we re-implement some baselines in the list  such as vanilla methods which means using AlexNet~\cite{krizhevsky2012imagenet} / ResNet50~\cite{hermans2017defense}, classification methods~\cite{zhu2017deep, wu2019deep}, natural image classification methods~\cite{zhou2016learning,fukui2019attention} and counterfactual generation methods~\cite{zhu2017unpaired,schlegl2017unsupervised,siddiquee2019learning}.

\textbf{Result Analysis.} 
As shown in Table~\ref{tab:comparison_all}, we achieve state-of-the art performance. We outperformed attention-based methods (Zhu~\cite{zhu2017deep}, ABN~\cite{fukui2019attention} and CAM~\cite{zhou2016learning}) largely by $4.9\%$ to $10.5\%$, multi-view method (Wu~\cite{wu2019deep}) largely by $4.7\%$ to $7.5\%$ and GAN-based methods( AnoGAN~\cite{schlegl2017unsupervised}, Fixed-Point GAN~\cite{siddiquee2019learning} and CycleGAN~\cite{zhu2017unpaired}) largely by $4.0\%$ to $11.5\%$.
Specifically, Zhu~\cite{zhu2017deep}, ABN~\cite{fukui2019attention} and CAM~\cite{zhou2016learning} take advantage of the attention mechanism. They all outperform the vanilla baseline. However, without exploiting the domain knowledge of mammograms, their performances are limited. 
Wu~\cite{wu2019deep} uses multi-view simple fusion. Better results compared with vanilla baseline indicate the bilateral information is useful. However, they are inferior to us since mammograms can not be pixel-to-pixel aligned. 
As to AnoGAN~\cite{schlegl2017unsupervised}, compared with the vanilla baseline, AnoGAN performs slightly worse in INBreast dataset than in the in-house dataset. We argue this is because there are relatively more sufficient healthy images in the in-house dataset, leading to better healthy modeling. However, they are still much lower than us due to suffering from various healthy patterns in mammogram.  
Fixed-Point GAN~\cite{siddiquee2019learning} and CycleGAN~\cite{zhu2017unpaired} achieve similar performances due to similar cycle consistency constraints. They outperform AnoGAN since they can make use of the image-level annotations. However, their performances are limited by suffering from the ill-posed translation on lesion removal. 

\begin{figure*}
\begin{center}
    \includegraphics[height=0.7\linewidth]{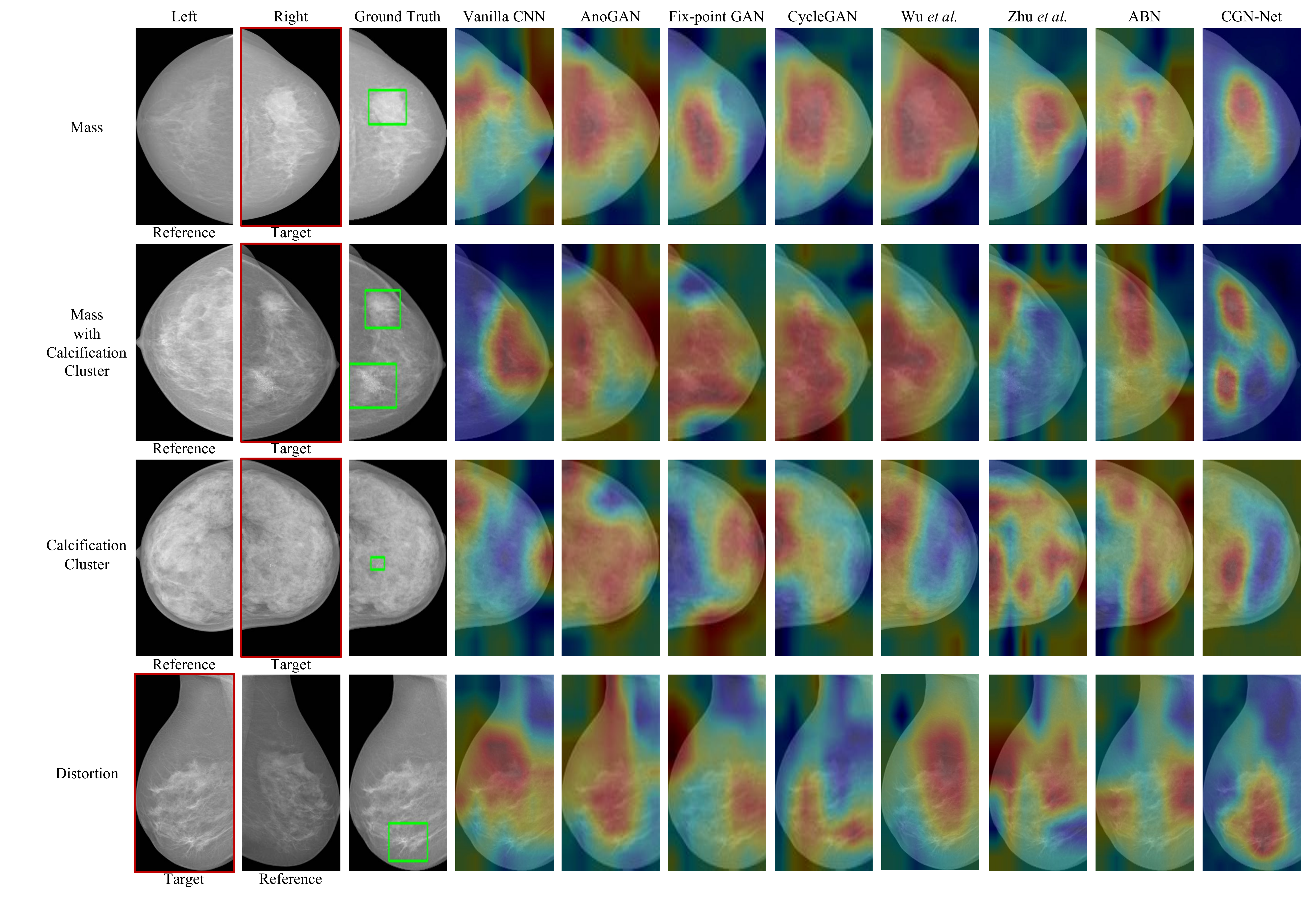}
\end{center}
    \centering
    \caption{Visualization of class activation maps of Vanilla CNN, AnoGAN~\cite{schlegl2017unsupervised}, Fixed-Point GAN~\cite{siddiquee2019learning}, CycleGAN~\cite{zhu2017unpaired}, Wu \textit{et al.}~\cite{wu2019deep}, Zhu \textit{et al.}~\cite{zhu2017deep}, ABN\cite{fukui2019attention} and CGN. Each row represents a pair of mammograms from bilateral breasts in the INBreast. The target containing lesions is bounded by a red rectangle. The ground truth bounding boxes are labeled by green rectangles. 
}
\label{fig:vis}
\end{figure*}

\textbf{Localization Evaluation} To verify whether the proposed model focuses on the lesion areas or not, we evaluate the localization error by CAM~\cite{zhou2016learning}. Same as~\cite{zhou2016learning}, we first calculate the CAMs based on the predicted category. Then to generate a bounding box from CAM, we segment the regions whose CAM value is larger than 20\% of the max CAM value and obtain the bounding box for the largest connected component in the segmentation map. We use the top-1 localization error as ILSVRC except for the intersection over union (IOU) threshold of 0.1, since our main concern is the classification performance, the precise localization is not necessary. As is shown in Table~\ref{tab:localization}, the proposed method obtains a localization error of 0.421 for masses and 0.455 for all lesions, outperforming other methods.

\textbf{Visualization} To verify the effectiveness of CGN in terms of learning lesion area, we visualize the class activation maps, as shown in Fig.~\ref{fig:vis}. 
We can see the asymmetry of lesions on bilateral images validates the bilateral asymmetric prior (the first three columns).The proposed CGN succeeds to focus on all lesions since it incorporates the bilateral symmetry prior. 
In contrast, the other methods show uneven results. In the first two cases, the other methods also show reasonable attention since the mass areas are highly different from the background. However, for the last two cases, the lesions are relatively indistinct. Thus it is quite challenging to find the lesions without bilateral information.

\subsection{Counterfactual Validation}
Since there are no ground truth images under counterfactual conditions, we validate the effectiveness and reasonableness of our generated counterfactual features in two aspects, the FID measurement and the further feature visualization, which are motivated by counterfactual evidence in~\cite{Besserve2020Counterfactuals}.

\begin{table*}[ht]
\centering
\begin{tabular}{c|c|c|c|c|c|c}
\hline
\textbf{Bilateral} & \textbf{$\mathbf{L}_{\mathbf{NE}}$} & \textbf{Triplet Loss}  & \textbf{AUC}(a) & \textbf{AUC}(b) & \textbf{AUC}(c) & \textbf{AUC}(d)\\
\hline
$\times$ & $\times$ & $\times$  & 0.820 & 0.827 & 0.780 & 0.697\\
SBF & $\times$ & $\times$ & 0.862 & 0.858 & 0.807 & 0.721\\
TF-GAN & $\times$ & $\times$  & 0.883 & 0.873 & 0.857 & 0.731 \\
BF-GAN & $\times$ & $\times$  & 0.860 & 0.842 & 0.849 & 0.720\\
AdaIN-GAN & $\times$ & $\times$  & 0.886 & 0.873 & 0.858 & 0.734\\
AdaIN-GAN & $\checkmark$ & $\times$ & 0.891 & 0.898 & 0.874 & 0.777\\
AdaIN-GAN & $\checkmark$ & Non-feedback &  0.873 & 0.863 & 0.858 & 0.741\\
$\times$ & $\checkmark$ & Feedback  & 0.837 & 0.851 & 0.836 & 0.716\\
AdaIN-GAN & $\times$ & Feedback  & 0.905 & 0.902 & 0.884 & 0.771\\
\hline
AdaIN-GAN & $\checkmark$ & Feedback  & \textbf{0.910} & \textbf{0.911} & \textbf{0.885} & \textbf{0.781}\\
\hline
\end{tabular}
\caption{AUC evaluation of ablation study on (a) INBreast dataset for mass classification with Alexnet; (b) INBreast dataset for mass classification with Resnet50; (c) INBreast dataset for mixed-lesion classification with Resnet50; (d) in-house dataset for mixed-lesion classification with Alexnet.} 
\label{ablation_all}
\end{table*}

\textbf{Counterfactual Visualization}
We visualize the target features, reference features, and generated counterfactual features in Fig.~\ref{fig:feature} to further verify the effectiveness of our counterfactual generation qualitatively. Since the three kinds of features are all with high dimension, we perform the max-pooling cross the channel dimension to generate the visualization heatmap for each of them. The heatmaps are shown in the last three columns respectively.
We can see that the activated lesion features in the target features marked by green rectangles disappear in the counterfactual features. While the counterfactual features in lesion-free areas are similar to the target features. This means that the proposed method can generate a healthy version of the target features, i.e., counterfactual features, effectively.

We also visualize the predicted location of lesions during the iterative training process to further verify the effectiveness of CGN in Fig.~\ref{fig:iterative}. With the process of iteration, the predicted location of lesions becomes more and more accurate.

    

    

\textbf{FID measurement}
To further evaluate the effectiveness of the generated counterfactual features, we calculate the mean FID \cite{haarburger2019multiparametric} to measure the feature distribution distances in the INBreast. The mean FID between the target and reference features is 56.15. The counterfactual-reference mean FID is 27.04. The target-counterfactual mean FID is 25.42 while the one after removing the lesion areas from ground truth is 0.60. By comparing the four distances to each other, we find the learned counterfactual features contain both reference information and target information in healthy areas.

\subsection{Ablation Study}
\label{sec:ablation study}
We evaluate some variant models to verify the effectiveness of each component. The ablative results in Table.~\ref{ablation_all} show that deleting or changing any of the components would lead to a descent of the classification performance. 
Specifically, naive bilateral features fusion also leads to a boosting of $2.4\%$ to $4.2\%$ over vanilla on performance. It proves the bilateral symmetric prior is quite helpful for malignancy classification. Meanwhile, the proposed prediction feedback mechanism outperforms the non-feedback largely by $4.8\%$. We explain that the classification module provides additional useful supervision for lesion localization, making learning more accurate and stable. For additional counterfactual constraint of negative embedding loss, we show that it improves the performance by $0.9\%$. Here are some interpretation for the variants:

\textbf{$\times$ in the first raw}: Vanilla single view netwwork.

\textbf{SBF}: Simple Bilateral features. The bilateral features are directly concatenated and fed into the fusion layer;

\textbf{TF-GAN}: Target-feature GAN. Replace AdaIN input by target features only;

\textbf{BF-GAN}: Bilateral-feature GAN. Replace AdaIN input by simple combination of bilateral features;

\textbf{Non-feedback}: Estimate lesion areas $\Omega$ by the areas with the largest target-counterfactual distance.


To further verify the effectiveness of the proposed adversarial loss and feedback triplet loss $\mathcal{L}_{FT}$, we applied two variants respectively:

\begin{table*}
\centering
\begin{tabular}{c|c|c|c|c}
\hline
 Methodology & AUC(a) & AUC(b) & AUC(c) & AUC(d)  \\
\hline
Variant (1) & 0.884 & 0.886 & 0.878 & 0.767 \\
Variant (2) & 0.860 & 0.863 & 0.850 & 0.739 \\
\hline
Proposed Method & \textbf{0.910} &  \textbf{0.911} & \textbf{0.885} & \textbf{0.781}\\
\hline
\end{tabular}
\caption{AUC evaluation on (a) INBreast dataset for mass classification with Alexnet; (b) INBreast dataset for mass classification with Resnet50; (c) INBreast dataset for mixed-lesion classification with Resnet50; (d) in-house dataset for mixed-lesion classification with Alexnet.}
\label{tab:comparison_constraint}
\end{table*}

\begin{table*}
\centering
\small
\begin{tabular}{c|c|c|c|c}
\hline
 Methodology & AUC(a) & AUC(b) & AUC(c) & AUC(d)  \\
\hline
SBF & 0.862  & 0.858 & 0.807 & 0.721 \\
GF & 0.865  & 0.862 & 0.812 & 0.726 \\
SFF & 0.864  & 0.862 & 0.813 & 0.724 \\\cline{1-5}
\hline
Proposed Method & \textbf{0.910} &  \textbf{0.911} & \textbf{0.885} & \textbf{0.781}\\
\hline
\end{tabular}
\caption{AUC evaluation of biliteral comparative experiments on (a)(b)(c)(d). We evaluate our method on four different experiment settings to illustrate our performance against other methods. The four settings are: (a) INBreast dataset for mass malignancy classification with Alexnet; (b) INBreast dataset for mass malignancy classification with Resnet50; (c) INBreast dataset for mixed-lesion malignancy classification with Resnet50; (d) in-house dataset for mixed-lesion malignancy classification with Alexnet.
}
\label{tab:comparison_bilateral}
\end{table*}

\textbf{Variant (1)}: As to the discriminator loss, we directly minimize the distance between counterfactual features $H_C^{\Omega}$ and reference features $H_R^{\Omega}$ in lesion areas. We still estimate the lesion areas ${\Omega}$ by the prediction feedback mechanism. 

Compared with the competing losses we used for discriminator and generator in our paper:
\begin{align}
\min_G \max_D \mathcal{L}_{AD}(G,D) := \log\left ( D\left ( H_R \right ) \right ) \notag\\ + \log\left ( 1 - D\left ( G(H_T,H_R) \right ) \right ).
\end{align}

We denote the modified discriminator loss and generator loss of variant (1) as:


\begin{equation}\label{eq:lossromega}
\mathcal{L}_{G}^{\Omega}=log\left ( 1-D\left ( H_C^{\Omega} \right ) \right )
\end{equation}
\begin{equation}\label{eq:lossdomega}
\mathcal{L}_{D}^{\Omega}=-log\left ( 1-D\left ( H_C^{\Omega} \right ) \right ) - log\left ( D\left ( H_R^{\Omega} \right ) \right )
\end{equation}

therefore we have the final losses:
\begin{equation}\label{eq:loss1_v1}
    \mathcal{L}_{1} = \mathcal{L}_{G}^{\Omega} + \mathcal{L}_{NE} + \mathcal{L}_{CLS}
\end{equation}
which are iteratively trained with $\mathcal{L}_{D}^{\Omega}$.

\textbf{Variant (2)}: As to the feedback triplet loss $\mathcal{L}_{FT}$, we design a variant feedback loss $\mathcal{L}_{FC}$ instead. We direct constraint the generated features $H_C^{\overline{\Omega}}$ in lesion-free areas to be similar to target features $H_T^{\overline{\Omega}}$. 

The $\mathcal{L}_{FC}$ is defined as:
\begin{equation}\label{eq:lossftc}
    \mathcal{L}_{FC} = d_{tc}
\end{equation}
where $d_{tc}$ is defined as Eq.~\eqref{loss:dtc};

Therefore we have the final losses:
\begin{equation}\label{eq:loss1_v2}
    \mathcal{L}_{2} = \mathcal{L}_{G} + \mathcal{L}_{NE} + \mathcal{L}_{FC} + \mathcal{L}_{CLS}
\end{equation}
which are iteratively trained with $\mathcal{L}_{D}$. The $L_G$ and $L_D$ are the generator loss and the discriminator loss respectively, as we used in the competing loss in $\min\limits_G \max\limits_D \mathcal{L}_{AD}(G,D)$.

\begin{figure*}
\begin{center}
    \includegraphics[height=1.2\linewidth]{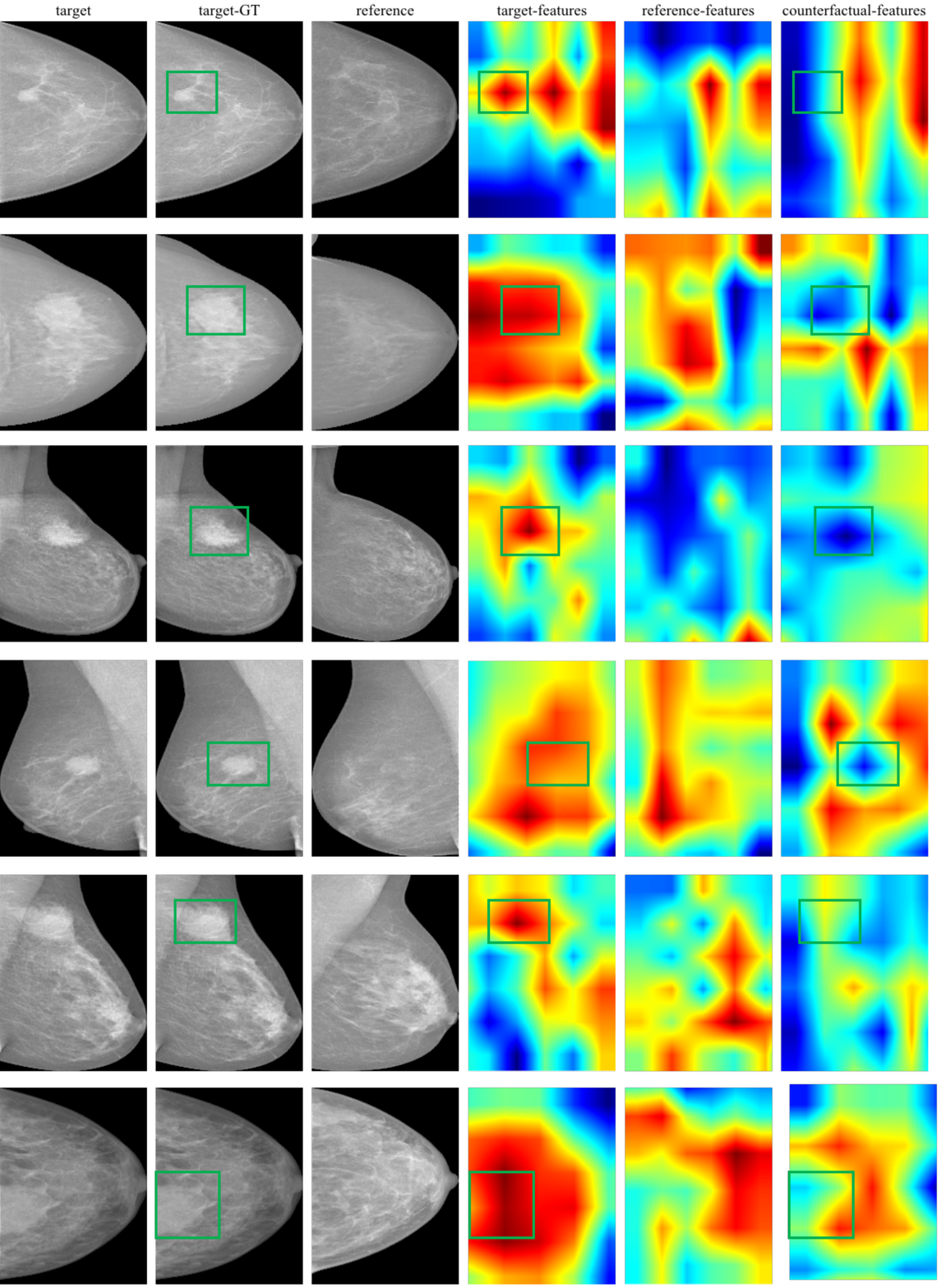}
\end{center}
    \centering
    \caption{\textbf{Visualization.} Left three columns: the target images, the target images with ground truth annotations which are marked by green rectangles on lesion areas, and reference images which are flipped horizontally for convenient comparison; Right three columns: feature maps of target images, feature maps of reference images, and feature maps of our generated counterfactual features. All visualized features are obtained by taking the maximum value of 256 channels. The green rectangles in each row mark the features in lesion areas before and after the counterfactual generation.}
    
    \label{fig:feature}
\end{figure*}

\begin{figure*}
\begin{center}
    \includegraphics[height=1.2\linewidth]{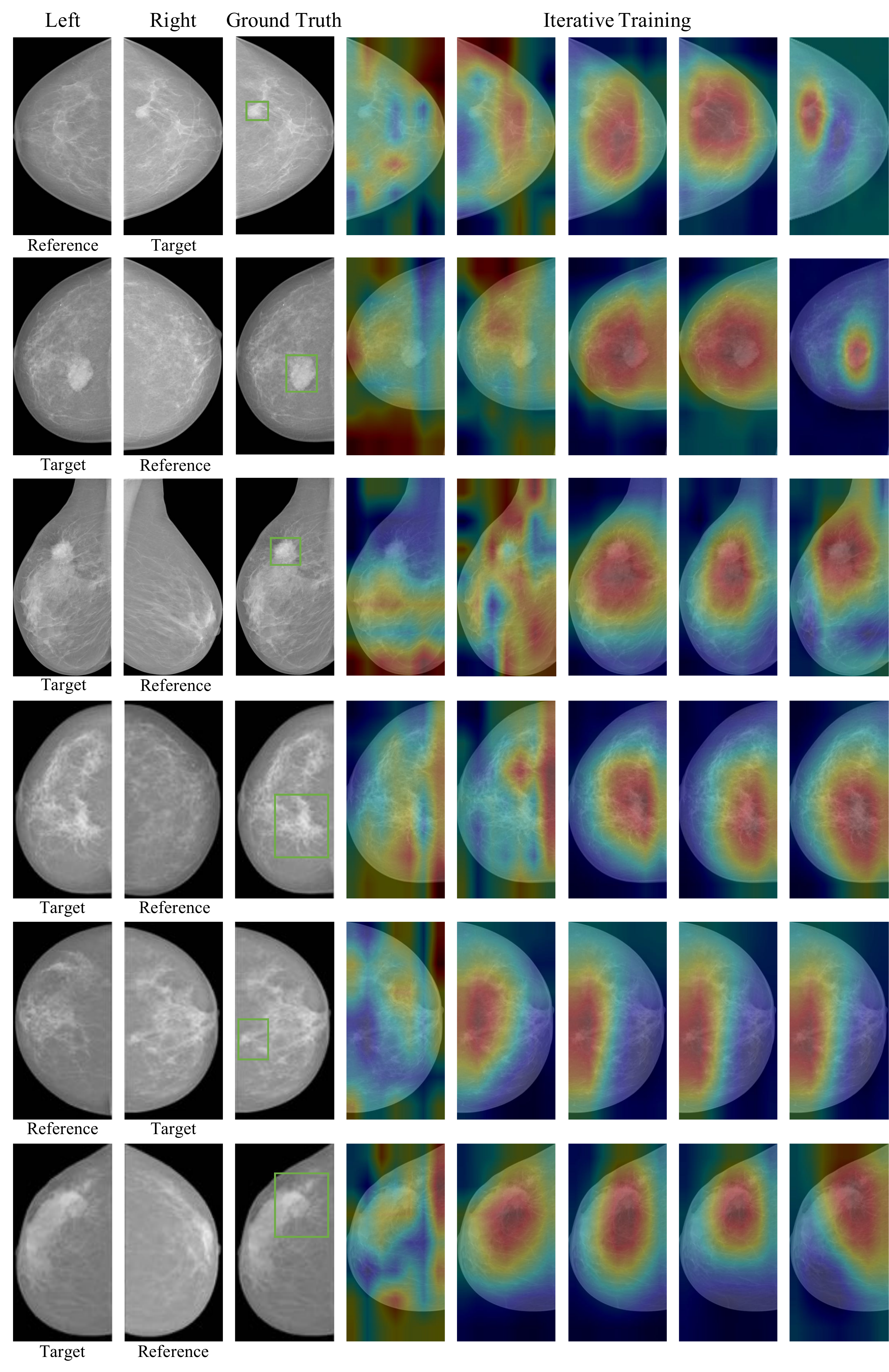}
\end{center}
    \centering
    \caption{\textbf{Iterative Training Process.} Left three columns: the images of the left side, the images of the right side, with being target or reference marked below, and the target images with ground truth annotations which are marked by green rectangles on lesion areas; Right five columns: the predicted location of lesions by CGN during training per ten epochs.}
    
    \label{fig:iterative}
\end{figure*}

The experimental results of the two variants against our proposed method are shown in Table.~\ref{tab:comparison_constraint}. We can see that modifying either the adversarial loss $\mathcal{L}_{AD}(G,D)$ or the feedback triplet loss $\mathcal{L}_{FT}$ would lead to a descent performance. 
We argue that our proposed losses are robust and effective. As we said that due to the pixel-to-pixel registration between bilateral images, we achieve counterfactual generation in feature level instead of image level. In practical experiments, we get $91.0\%$ of feature level which is higher than $90.6\%$ of image level, verifying the performance of the feature generation. Moreover, the training speed of the former is more faster than the latter with 6.6 s/epoch v.s. 23.5 s/epoch.

\subsection{Bilateral Analysis}
For bilateral analysis, we re-implement some interesting modules used in recent papers.

\textbf{SBF}: As mentioned in ablation study, Simple Bilateral Features. \emph{e.g.,} Kim~\textit{et al.}~\cite{kim2016latent} applied in ToMO;

\textbf{GF}: Gated fusion in SBF. Learning more weights for asymmetric enhancement based on SBF~\cite{liu2019unilateral};

\textbf{SFF}: Simple Four-view features fusion. Ensembling cross-view and contralateral-view simply~\cite{wei2011computer};

SFF and GF can be seen as variants of SBF. As shown in Table.~\ref{tab:comparison_bilateral}, SFF and GF slightly outperform SBF for using more information but are inferior to our proposed method for naive use of view-wise information. 
Both of them share the similar disadvantage with SBF: even for healthy breasts, bilateral mammograms are only roughly symmetric but not pixel-to-pixel, the similarity of bilateral features cannot be guaranteed. While our method uses the symmetric prior by counterfactual generation with an improved GAN. Therefore, our method suffers less from these problems and leads to better results.

\section{Conclusion}
In this paper, we propose a novel approach called bilateral asymmetry guided Counterfactual Generating Network (CGN) to improve the mammogram classification performance. The proposed method performs the counterfactual generation by exploiting the symmetric prior effectively. Experimental results indicate that the proposed CGN achieves state-of-the-art results in both public and in-house datasets. Our work can be referred as the showcase of exploiting symmetric prior, which widely holds in many human organs,\emph{e.g.,} brains, eyes, skeletal structures, and kidneys. Therefore, we believe that the generalization ability of our method on corresponding medical imaging problems, the efforts of which will be left in future work.

\appendices

\section{Proof of Theorem 3.1}

\begin{lemma}\label{lemma}
If the the causal graph $\mathcal{G}$ satisfies that the common factor $C$ influences the bilateral variables simultaneously, then,
\begin{equation}
\begin{aligned}
    & f_{Y_T}(C, \cdot) = f_{Y_R}(C, \cdot)\\
    & f_{H_T}(C, \cdot) = f_{H_R}(C, \cdot)\\
    & f_{X_T}(C, \cdot) = f_{X_R}(C, \cdot)\\
\end{aligned}
\end{equation}
\end{lemma}
Lemma~\ref{lemma} shows that the causal factor $C$ influences the bilateral mammograms in equal function relationship. 

\begin{proof}[Proof of Theorem~\ref{thm1}]

\textbf{Proof} of Eq.~\eqref{eq1}: 

\begin{equation}
\begin{aligned}
    & P(H_{{T}_{(Z_T^{\Omega}=0)}}^{\Omega}=h|H_T^{\Omega}=h_t, Z_T^{\Omega}=1) \\
        & = \int_{c}P(H_{{T}_{(Z_T^{\Omega}=0)}}^{\Omega}=h|C=c)P(C=c| H_T^{\Omega}=h_t, Z_T^{\Omega}=1)dc \\ 
        & = \int_{c}P(H_T^{\Omega} = h|C=c, Z_T^{\Omega}=0)P(c|H_T^{\Omega}=h, Z_T^{\Omega}=1)dc\\
        & = \int_{c}P(H_R^{\Omega} = h|C=c, Z_R^{\Omega}=0)P(c|H_T^{\Omega}=h, Z_T^{\Omega}=1)dc\\
        & = P(H_{{R}_{(Z_R^{\Omega}=0)}}^{\Omega}=h|H_T^{\Omega}=h_t, Z_T^{\Omega}=1),\\
        & = P(H_{R}^{\Omega}=h_r|H_T^{\Omega}=h_t, Z_T^{\Omega}=1)
\end{aligned}
\end{equation}
where the first equation is due to that the $c$ is the only parent node of $H_{{T}_{(Z_T^{\Omega}=0)}}^{\Omega}$; the second equation is according to Markov condition that $H_{{T}_{(Z_T^{\Omega}=0)}}^{\Omega}|C = H_T|C,Z_T^{\Omega}$, the third equation is due to the symmetric prior.


\textbf{Proof} of Eq.~\eqref{eq2}: 
Since in the lesion-free areas, there are $Z_T^{\overline{\Omega}}=0$, the probabilities are derived by the actual hidden features $H_T=h_t$ directly, i.e.,
\begin{equation}
\begin{aligned}
& P(H_{{T}_{(Z_T^{\overline{\Omega}}=0)}}^{\overline{\Omega}}=h|H_T^{\overline{\Omega}}=h_t, Z_T^{\overline{\Omega}}=0) \\
& = P(H_{T}^{\overline{\Omega}}=h_t|H_T^{\overline{\Omega}}=h_t, Z_T^{\overline{\Omega}}=0)
\end{aligned}
\end{equation}


\end{proof}

\clearpage
{\small
\bibliographystyle{plain}
\bibliography{egbib}
}

\end{document}